%% file: aia.tex
\pdfoutput=1

\documentclass{article}

\usepackage{times}
\usepackage{graphicx} 
\usepackage{subfigure} 

\usepackage{natbib}

\usepackage{algorithm}
\usepackage{algorithmic}
\usepackage{amssymb}
\usepackage{amsmath}
\usepackage{amsthm}
\newtheorem{theorem}{Theorem}
\newtheorem{corollary}{Corollary}
\newcommand{\bettercomment}[1]{\hfill$\blacktriangleright$ #1}
\newcommand{\abr}[1]{\textsc{#1}}

\usepackage{hyperref}

\usepackage[accepted]{icml2016} 

\usepackage{breqn}

\icmltitlerunning{Active Information Acquisition}

\begin{document} 

\twocolumn[
\icmltitle{Active Information Acquisition}


\icmlauthor{He He}{hhe@umiacs.umd.edu}
\icmladdress{University of Maryland, College Park, MD 20742 USA}
\icmlauthor{Paul Mineiro, Nikos Karampatziakis}{\{pmineiro,nikosk\}@microsoft.com}
\icmladdress{Microsoft CISL, 1 Microsoft Way, Redmond, WA 98052 USA}

\icmlkeywords{reinforcement learning, learning to search, sequential prediction}

\vskip 0.3in
]

\begin{abstract} 
We propose a general framework for sequential and dynamic acquisition
of useful information in order to solve a particular task. While our
goal could in principle be tackled by general reinforcement learning,
our particular setting is constrained enough to allow more efficient
algorithms. In this paper, we work under the Learning to Search framework
and show how to formulate the goal of finding a dynamic information
acquisition policy in that framework. We apply our formulation on two
tasks, sentiment analysis and image recognition, and show that the learned
policies exhibit good statistical performance. As an emergent byproduct,
the learned policies show a tendency to focus on the most prominent
parts of each instance and give harder instances more attention without
explicitly being trained to do so.
\end{abstract} 

\input{intro}
\input{model}
\input{experiments}

\section{Conclusion}

In this paper we showed how to formulate the task of learning to acquire
information for solving a particular problem inside the L2S paradigm. We
proposed a computationally simple reference policy (that has access to the
training labels) and used imitation learning to compete with it, avoiding
the difficulties of more general reinforcement learning techniques. We
also proposed a loss function that explicitly balances the trade-off
between the task loss and the cost of information acquisition. The
effect of minimizing this trade-off is the learned policies focus on the
prominent parts of the input and spend more effort on examples that are
harder to classify.

We believe that much of the existing work on dynamic information gathering
can leverage imitation learning and the L2S framework instead of falling
back to more general reinforcement learning techniques. For example, in
early classification of time series, the future is eventually observed,
which facilitates constructing a reference policy at training time.
Therefore, fruitful directions for future work include adapting and
extending the ideas we presented in this paper to other domains where
the active collection of information can be simulated at training time.


 
\bibliography{../aia}
\bibliographystyle{icml2016}

\end{document}

%% file: intro.tex
\section{Introduction}
\label{introduction}

In the supervised learning framework, a learning algorithm is given
example input-output pairs with which to model the desired behaviour.
However real life autonomous agents must dynamically acquire the
information they need for decisions based upon goals and current
knowledge.  Thus the information required varies across different
instances of the problem.  Furthermore, given a time or expense budget,
an algorithm can attempt to balance a trade-off between cost of acquiring
information (and reasoning about it) and quality of the result.  These
considerations apply both to understanding psychophysical phenomena
such as planning saccades~\cite{araujo2001eye} and to developing
practical solutions to problems such as early classification of time
series~\cite{dachraoui2015early}.

We propose a general-purpose framework that sequentially processes
the input, adaptively selects parts of it, and combines the acquired
information to make predictions.  Our framework can be applied to any
base model (e.g. generalized linear models, neural networks) with any
information unit (e.g. features, feature groups or pieces of raw input).

Specifically, given a prediction task, our goal is to learn a \emph{task
predictor} and an \emph{information selector}.  The task predictor takes
information acquired by the selector and generates outputs defined by
the specific task, such as object classes for image classification.
The information selector acquires pieces of information based on past
information and intermediate predictions given by the task predictor.
We model this dynamism as a sequential decision-making process as
shown in Figure~\ref{fig:testtime}, where we make a decision about
which information to acquire at each step. The process stops when the
model decides that enough information has been obtained and outputs its
final prediction.  We use the Learning to Search (L2S)~\cite{daume14l2s}
framework, which casts searching for a good policy as an imitation
learning problem: at training time we have access to (can simulate) a
reference policy which is possibly accessing the training labels, and the
goal is to induce a policy that mimics the reference policy at test time.

Our contribution is an active information acquisition model that is
flexible enough to apply to different tasks with different predictors
and information units.  Our model explicitly minimizes a user-specified
trade-off between cost on information and quality of prediction.
We quantify the trade-off as the loss function for L2S. As there are no
constraints on the loss function, our model can accommodate different
types of loss defined by a task and even loss functions that do not
decompose nicely over the search space.\footnote{In some applications,
the cost of a piece of information may depend on whether another piece
of information has been acquired or not.} The L2S framework additionally
requires the specification of a search space, and a reference policy.
Our formulation for these ingredients in the case of active information
acquisition is detailed in Section~\ref{sec:l2s}.

We evaluate our algorithm on a sentiment analysis task with a bag-of-words
predictor, and an image classification task with a convolutional
neural network (CNN).  Our algorithm achieves better results than
static information selection baselines on both tasks.  Additionally,
we show that the dynamic selector learns to acquire more information
for difficult examples than easy examples.

\begin{figure*}
\begin{tabular}{cc}
\begin{minipage}{0.45\textwidth}
\begin{center}
\includegraphics[scale=0.5]{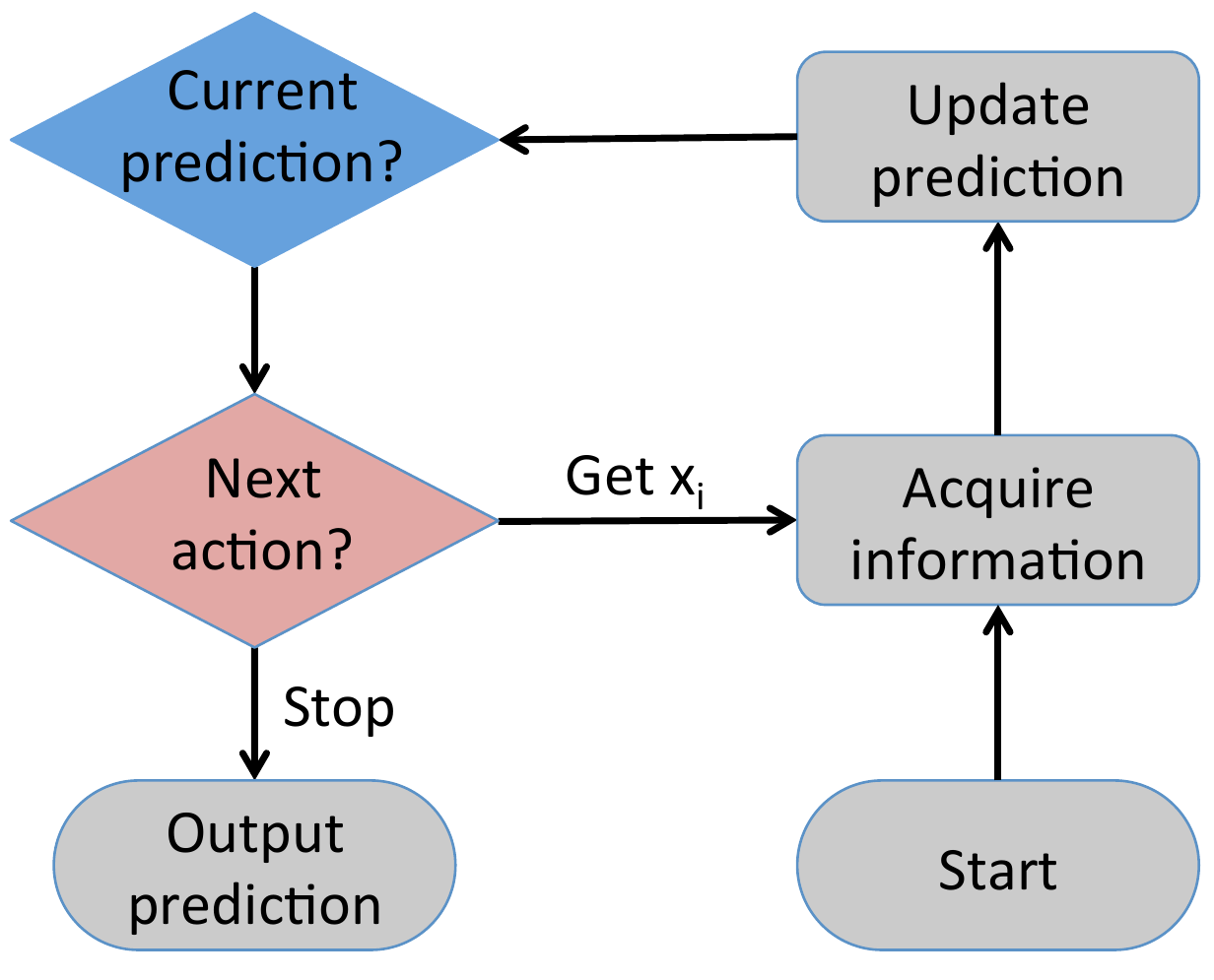}
\end{center}
\end{minipage} &
\begin{minipage}{0.5\textwidth}
\begin{algorithm}[H]
\caption{\textsc{Predict} ($\mathcal{X}, \mathcal{X}'_0, h, \pi$)}
\label{alg:testtime}
\begin{algorithmic}[1]
\FOR{$t=0$ to $|A(\mathcal{X}'_0)|$}
\STATE $\hat{y}_t  \gets h(\mathcal{X}'_{t})$ \bettercomment{Intermediate prediction}
\STATE $a_t \gets \pi(\mathcal{X}'_{t}, \hat{y}_t)$ \bettercomment{Select information} 
\IF{$a_t=$ \abr{stop}}
\STATE return $\hat{y}_t, \mathcal{X}'_{t}$ \bettercomment{Early stop, return terminal state}
\ELSE 
\STATE $\mathcal{X}'_{t+1} \gets \mathcal{X}'_t\,\bigcup\,\{x_{a_t}\}$ \bettercomment{Add new information}
\ENDIF 
\ENDFOR
\end{algorithmic} 
\end{algorithm} 
\end{minipage}
\end{tabular}
\caption{Information acquisition at test time. \emph{Left}: a flowchart
of our algorithm. The blue diamond and the red diamond represent the
task predictor and the information selector respectively. \emph{Right}:
pseudocode of the execution.}
\label{fig:testtime}
\end{figure*}

\section{Related Work}
\label{relatedwork}

The topic of learning information gathering policies has received
much interest lately. Many of the proposals in this space however use
general Markov decision process (MDP) techniques, which are sufficient
but perhaps not necessary given the constrained, deterministic world
of sequential selection.  

\citet{kanani2012selecting} learn a policy for filling in missing entries
in a knowledge base, where the actions are querying a search engine,
downloading a page or extracting information from a page. For learning
the policy, they use temporal difference Q-learning and briefly mention
potentially more efficient techniques but always within the general MDP
learning framework.

Our work is closest to \citet{dulac2011text,dulac2014image}, who explored
sequential text and image classification with results analogous to
our experiments.
The authors proposed reinforcement learning techniques with adaptation
to different tasks, while our approach is general and efficient enough to
apply to a range of problems.  More importantly, when the complete inputs
are available (but hidden to the learning algorithm), we can compute
a good reference policy and incorporate it into L2S through imitation
learning for more efficient training.  Another important distinction is
that they use a single policy as both the task predictor and information
selector.  This formulation has a larger search space compared to ours
and does not leverage pre-training of the task predictor.  In addition,
it might face difficulty in complex domains where the predictor and the
selector need different function classes.

\citet{mnih2014recurrent} explored sequential visual inspection for
image classification, with results analogous to our image classification
experiment.  Important technical differences are the use of policy
playouts and the specific use of recurrent neural networks.  Our approach
admits the use of recurrent neural networks for either the predictor
or selector components, but does not require it.  In other words, the
model is a special case of our framework with particular choices for the
predictor and selector components.  Furthermore, that work demonstrated
improved aggregate performance with diminishing returns for fixed budgets
of sensor utilization, but do not consider policies which make a variable
number of sensory measurements.  Similar comments apply to the recent
visual attention work of~\citet{ba2015learning}.

Our loss function quantifies the information-accuracy trade-off.
Any approach leveraging general reinforcement learning can optimize such
a loss: nonetheless, the prior art above did not do so.  This trade-off
can be critical in practical applications, e.g., minimum cost spam
filtering~\cite{blanzieri2008survey}, and has been treated explicitly
in the case of classifier cascades~\cite{ChenXuWeinbergerChapelle2012}
and early classification of time series~\cite{dachraoui2015early}.

Our work is also related to dynamic feature selection.
\citet{hhe2012coaching} used DAgger~\cite{ross11dagger} to select
features sequentially with a loss function similar to ours.  DAgger is
a specific implementation of L2S that does not consider cost of errors,
and we observe degrading results with uniform cost in our experiments. In
addition, they consider information selection on the feature level only.
In \citet{gao11active}, classifiers are selected dynamically based on
their value of information under a probabilistic framework.  Again, they
consider a particular form of information---observation presented as
classification results---while we embrace a broader class of information.

\citet{poczos2009learning} consider the problem of learning a stopping
policy to maximize expected reward per unit time given a fixed sequence
of classification strategies with variable associated temporal costs.
A key distinction from this work is that the sequence of classification
strategies is fixed, rather than trained jointly with the stopping policy.

%% file: model.tex
\section{Active Information Acquisition Framework}
\label{sec:framework}
We assume that the input data $x$ can be decomposed to multiple parts,
such that $\mathcal{X}=\bigcup_{i=1}^{n}\{x_i\}$, where $n$ is the number
of parts. We denote a partial input by $\mathcal{X}'$, where $\mathcal{X}'
\subseteq \mathcal{X}$.  It is straightforward to extend the framework
to input data with variable number of parts per example but we do not
for ease of exposition.

Our framework consists of a task predictor $h$ and an information selector
$\pi$, which interact as shown in Figure~\ref{fig:testtime}.  Both $h$
and $\pi$ access the input through feature maps, which we omit here to
simplify notation.

The task predictor $h$ transforms a partial input into a prediction $\hat
y = h (\mathcal{X}')$, e.g., for a multiclass problem the task predictor
can take a partial input and produce a distribution over the labels.

The information selector is a policy $\pi \in \Pi$ that takes as
input a state, which summarizes the information collected so far and
any previous prediction(s), and outputs an action to take next: $a =
\pi(s)$.  The actions are (a) to acquire a new piece of information
(and to specify which one) and (b) to stop and output the current
prediction.  The complete set of actions is $A = \{ 1, \ldots, n
\}\,\bigcup\,\{\text{\abr{stop}}\}$.  Added information is excluded
from the action set, and we use $A(\mathcal{X}')$ to denote the action
set specific to $\mathcal{X}'$, including non-selected information
and \abr{stop}.

Our framework allows task-dependent choices of the learning components
$h$ and $\pi$.  However, because these components must be able to work
with any subset of input parts, idiosyncratic changes are required for
different choices of $h$.  Handling missing and incomplete data is an area
with an extensive literature.  For our experiments, we find the following
simple strategy effective: augmenting the input with an additional binary
variable per part indicating whether or not a part has been observed
and setting the feature values for the unobserved parts to 0.

\section{Learning to Search for Information}
\label{sec:l2s}

Our framework builds on top of the Learning to Search~\cite{daume14l2s}
(L2S) paradigm, which allows us to jointly train the (interdependent)
information selector and the task predictor via a reduction to online
cost-sensitive classification.

The L2S algorithm requires three components: a \emph{search space}
which defines states, actions, and transitions, a \emph{loss function}
to evaluate the result given an action sequence, and a \emph{reference
policy} that suggests good actions given any state during training.
Essentially, L2S learns a policy that \emph{imitates} the reference
policy, assuming that the reference policy attains good performance.
Below we describe details of each component in our setting and the
training algorithm.

\begin{figure}[t]
\centering
\includegraphics[width=0.45\textwidth]{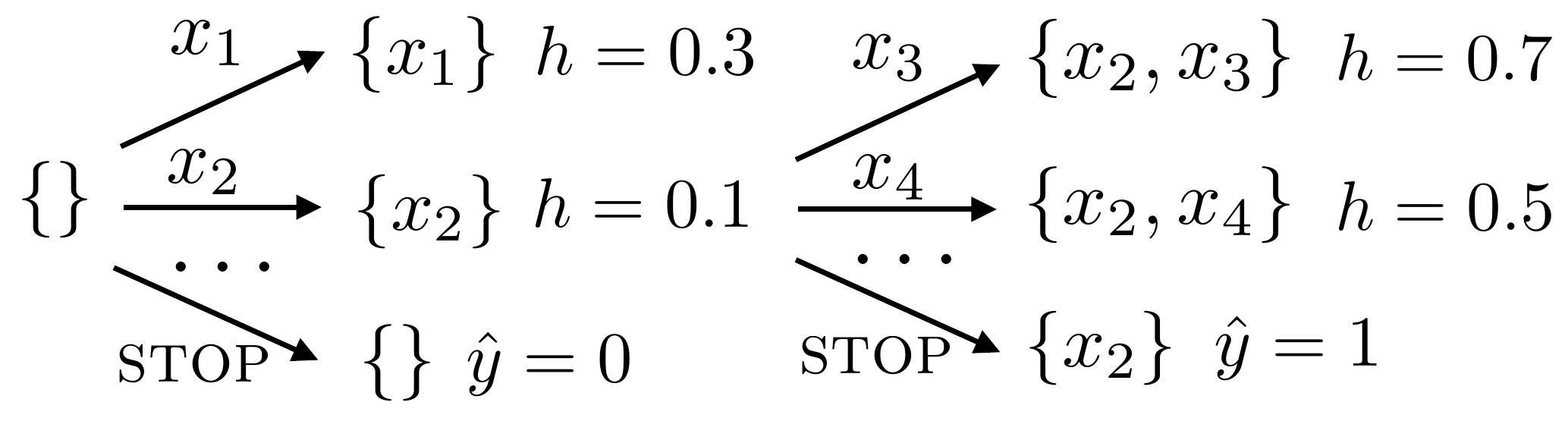}
\caption{An illustration of the search space. Starting with an empty set, information is acquired sequentially and intermediate predictions are made at each step.}
\label{fig:search_space}
\end{figure}

\begin{algorithm*}
\caption{L2S Joint Training}
\begin{algorithmic}[1]
\STATE Initialize $\pi_1$, $h \gets h_0$ \bettercomment{Pre-trained task predictor}
\FOR{$i=1$ to $N$ (loop through examples)}
  \STATE $\mathcal{D} \gets \emptyset$, $\mathcal{X}'_i \gets \emptyset$ 
  \WHILE{$A(\mathcal{X}'_i) \neq \emptyset$}
    \STATE $\hat{y}_{i}  \gets h(\mathcal{X}'_{i})$ \bettercomment{Intermediate prediction}
    \FOR{$a \in A(\mathcal{X}'_i)$ (one-step deviation)}
      \STATE $\hat{y}_{i,a}, \mathcal{X}'_{i,a} \gets$ Execute $a$ and roll out until termination as in \abr{Predict} 
      \STATE $c(a) = \ell(\hat{y}_{i,a}, y_i, \mathcal{X}'_{i,a})$ \bettercomment{Evaluate end loss}
    \ENDFOR
    \STATE $c(a) \gets c(a) - \min_a c(a)$  
    \STATE $\mathcal{D} = \mathcal{D}\,\bigcup\,\{c, (\mathcal{X}'_i, \hat y_i)\}$ \bettercomment{Collect example}
    \STATE $a \gets \pi(\mathcal{X}'_{i}, \hat{y}_{i})$ \bettercomment{Execute current policy} 
    \IF{$a=\text{\abr{stop}}$}
      \STATE Update $h$ with $(\mathcal{X}'_i, \hat{y}_i)$ \bettercomment{Fine-tune}
      \STATE {\bf break}
    \ELSE
      \STATE $\mathcal{X}'_{i} \gets \mathcal{X}'_i\,\bigcup\,\{x_{i,a}\}$ 
    \ENDIF
  \ENDWHILE
  \STATE $\pi_{i+1} \gets$ Update $\pi_i$ with $\mathcal{D}$ \bettercomment{Train policy}
\ENDFOR
\STATE Return the average policy $\pi$ of $\pi_1,\ldots,\pi_N$
\end{algorithmic}
\label{alg:l2s}
\end{algorithm*}

\paragraph{Search space}
Our state is a tuple of a partial input and an intermediate prediction:
$s=(\mathcal{X}', \hat y)$.  The action set for $\mathcal{X}'$ is
$A(\mathcal{X}')$, which is defined in Section~\ref{sec:framework}.
We do not ask for the same piece of information more than once by
disallowing actions corresponding to observed parts.  This restriction
is not necessary in other scenarios, such as a robot learning to
act in a dynamic environment where the same part of the world may
change over time.  An illustration of the search space is shown
in Figure~\ref{fig:search_space}.  After an action is taken,
the current state transitions to a new one deterministically by
adding the new information or terminating the process, as shown in
Algorithm~\ref{alg:testtime}, line 4--8.

\paragraph{Loss function}
To learn a trade-off between the amount of information and the quality
of the prediction, we define the loss function as
\begin{equation}
\ell(\hat{y}, y, \mathcal{X}') = \ell_{\text{task}}(\hat{y}, y) + \lambda \cdot \mathcal{C}(\mathcal{X}').
\end{equation}
Here $\ell_{\text{task}}$ is the loss function defined by the
task, which does not have to be convex, e.g., 0-1 loss, squared
loss.  $\mathcal{C}$ is the cost function of information.  In our
experiments, we set $\mathcal{C}=|\,\cdot\,|/n$, which computes the
percentage of parts acquired.  However, an arbitrary function of
$\mathcal{X}'$ can be used for acquisition cost, e.g., for variable
feature cost~\cite{ChenXuWeinbergerChapelle2012} or nonuniform cost
of delay~\cite{dachraoui2015early}.  We use $\lambda$ to control the
penalty on acquiring more information.  By varying $\lambda$ we can
construct a Pareto curve of cost vs. loss.

Since we compute intermediate predictions, the loss function can be
applied to results at any time step.  We call the loss at the end the
\emph{terminal loss} and those at earlier time steps the \emph{immediate
loss}, and our goal is to learn policies that minimize the expected
terminal loss.

\paragraph{Reference policy}
We use a greedy reference policy $\pi^\ast$ that always chooses the next
piece of information that yields the lowest immediate loss.  Formally,
\begin{equation*}
\pi^\ast(s_t) = \arg\min_a \ell(\hat{y}_t, y, \mathcal{X}'_t\,\cup\,\{x_a\}),
\end{equation*}
where $\{ x_{\textsc{stop}} \} \doteq \emptyset$.  As the performance of L2S
depends much on the quality of the reference policy, we analyze in
Section~\ref{sec:analysis} when a greedy policy is optimal and how
suboptimality affects the result. We have also verified that this policy
is performing well on the tasks in Section~\ref{sec:experiments}, in
fact leading the learned policy by a large margin.  Unlike our learned
policy however, the reference policy makes use of the training label
and therefore cannot be used at test time.

\paragraph{Joint Training}
During training, L2S calls the \textsc{Predict} function
(Algorithm~\ref{alg:testtime}) many times to explore different action
sequences and to discover the ones that have a low terminal loss,
similar to other reinforcement learning techniques.  However, with a
reference policy, L2S can explore the search space more efficiently
by initially focusing on areas close to the action sequences generated
by the reference policy and gradually deviating away by following the
learned policy~\cite{daume14l2s}.

We show the training procedure in Algorithm~\ref{alg:l2s}.  For each
example, we collect a set of cost-sensitive multiclass examples, where
class labels correspond to actions.  First an initial trajectory is
generated (roll in) by the current learned policy $\pi_i$,\footnote{We can
also roll in with a mixture of the reference policy and the learned policy
and gradually decrease the mixing weight of the reference policy. We did
not observe significant difference by using a mixture roll-in policy.}
then from the arrived state, the reference policy is executed until the
terminal state (roll out) to derive the terminal loss of each action.
The cost assigned to an action in a given state is the difference between
its loss and the minimum loss for the state (Algorithm~\ref{alg:l2s},
line 10).  Rolling in with the learned policy guarantees that states
of the collected examples are representative of states encountered at
test time.  Given tuples of state, action and loss as training examples,
the policy learning problem is reduced to standard cost-sensitive
multiclass classification.

We assume that an initial task predictor is given and intermediate
predictions are generated by calling it.  To initialize a task predictor
beforehand, we pre-train one on a small portion of the training data,
e.g., by using randomly sampled subsets of parts.  This pre-training
distribution is presumably unlike the one induced by a mature selector.
To mitigate this, we fine-tune the task predictor during training with
inputs generated by the information selector after each update (line
13--16 in Algorithm~\ref{alg:l2s}).\footnote{In practice, fine-tuning may
happen after some iterations when the selector is relatively stable.}
In other words, we adjust the task predictor $h$ to reduce the loss of
each intermediate prediction on the partial input sequences generated
by the selector $\pi$.

\begin{figure*}
\begin{tabular}{cc}
\begin{minipage}{0.5\textwidth}
\begin{center}
\includegraphics[width=0.9\columnwidth]{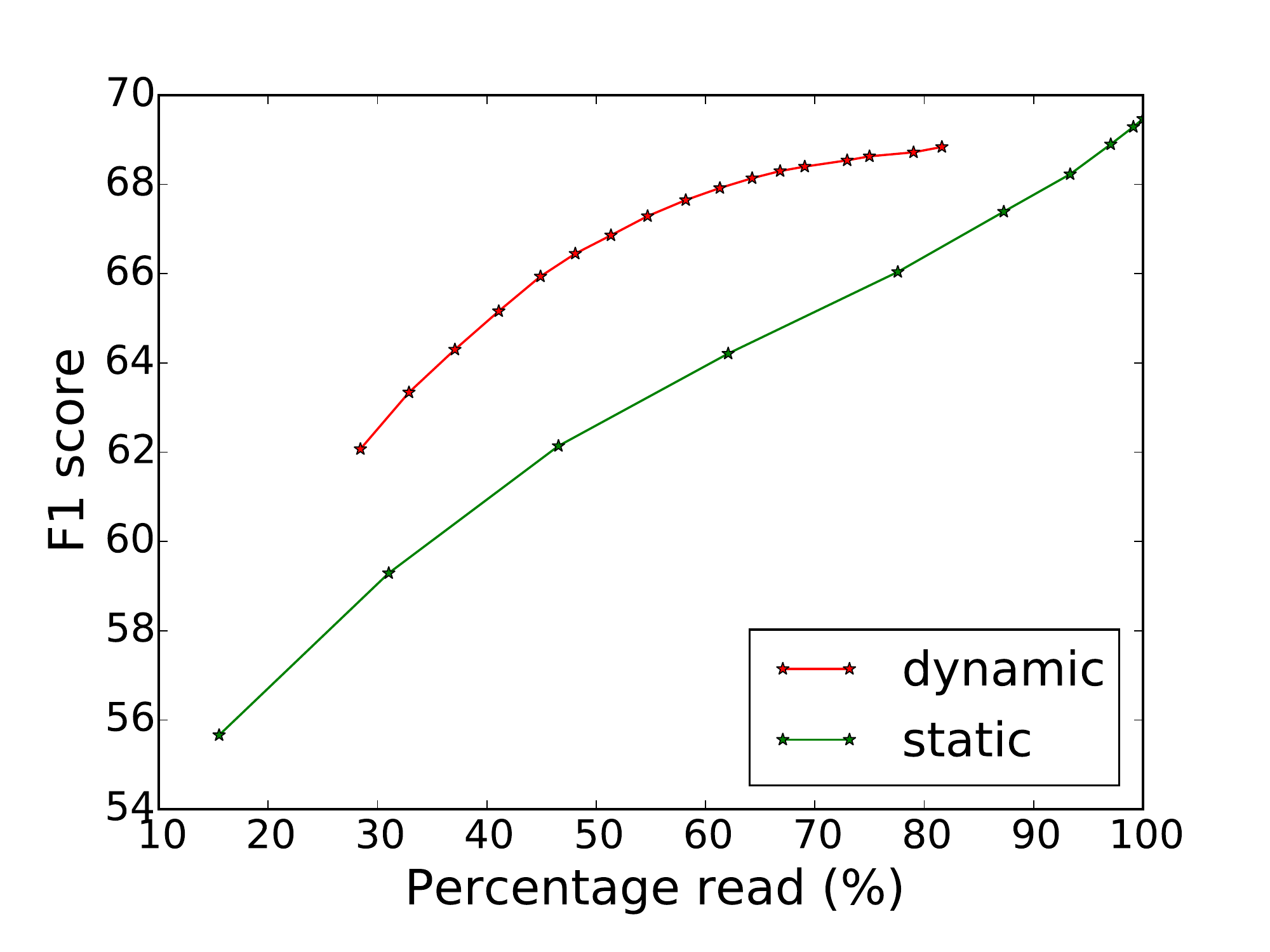}
\end{center}
\end{minipage}
&
\begin{minipage}{0.5\textwidth}
\begin{center}
\hspace{-3em}
\includegraphics[width=0.9\columnwidth]{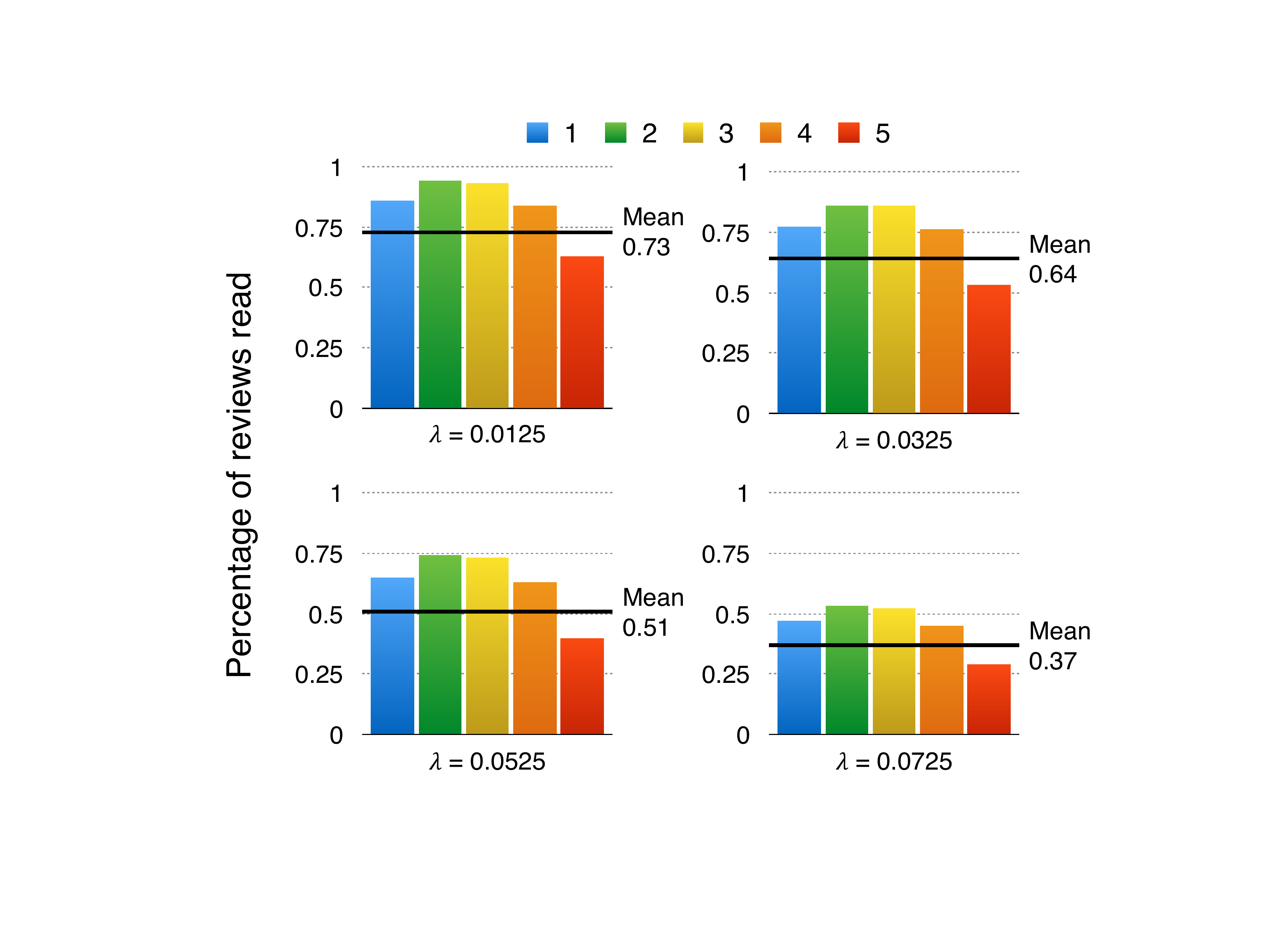}
\hspace{3em}
\end{center}
\end{minipage}
\end{tabular}
\caption{TL;DR performance on test data. \emph{Left}: comparison between the Pareto frontiers of AIA (dynamic) and static selection. \emph{Right}: Average fraction of sentences read as a function of (unobserved) rating, for a particular $\lambda$.}
\label{fig:tldr-pareto}
\end{figure*}

\section{Analysis}
\label{sec:analysis}
We now analyze the quality of the information selector returned
by Algorithm~\ref{alg:l2s}.  As L2S minimizes loss relative to the
reference policy, we measure performance of the learned policy by regret
to $\pi^\ast$.  We first present the regret guarantee of L2S, then extend
the result to our setting of information selection.

The loss of a policy $J(\pi)$ is defined as the expected terminal loss,
and the expectation is taking over distribution of the states induced
by running $\pi$.  We use $Q^\pi(s,a)$ to represent the terminal loss of
executing action $a$ in state $s$ and then following policy $\pi$ until
the terminal state.  We denote by $d_\pi^t$ the distribution of states
at step $t$ when running policy $\pi$ and $d_\pi=\frac{1}{T}\sum_{t=1}^T
d_\pi^t$, where $T$ is the horizon length, namely the maximum number of
parts of the input.  Thus we have
\begin{equation*}
J(\pi) = \mathbb{E}_{s \sim d_\pi} [Q^\pi(s,\pi(s))] ,
\end{equation*}
Henceforth, we use $Q^\pi(s,\pi')$ as a shorthand for $Q^\pi(s,\pi'(s))$.

L2S has the following regret guarantee:
\begin{theorem}
When using a no-regret cost-sensitive learner, 
the policy returned by Algorithm~\ref{alg:l2s} after $N$ steps satisfies
\begin{equation*}
J(\pi) - J(\pi^\ast) \le T\epsilon_{N},
\end{equation*}
where $\epsilon_{N}$ is defined as
\begin{equation*}
\frac{1}{NT}\sum_{i=1}^N\sum_{t=1}^T\mathbb{E}_{s_t \sim d_\pi^t} \left[ 
Q^{\pi^\ast}(s_t,\pi) - \min_a Q^{\pi^\ast}(s_t,a)
\right] .
\end{equation*}
\label{thm:l2s}
\end{theorem}
In words, the regret is bounded by the expected difference in
cost-to-go of the reference policy induced by a suboptimal action, and
increases linearly with the sequence length.  Readers are referred to
\citet{daume15lols} for the proof.

Now we specify the bound in our setting.  First we define suboptimality
of a reference policy.  Starting from any state, if the optimal
policy achieves terminal loss $L_{\text{opt}}$ , a reference policy
with suboptimality $\alpha$ achieves a loss no larger than $\alpha
L_{\text{opt}}$ ($\alpha \ge 1$).

Notice that the Q-values in $\epsilon_N$ differ only when a classification
error occurs.  We denote the classification error of a
policy $\pi$ as $\epsilon_{c} = \frac{1}{T}\sum_{t=1}^T \mathbb{E}_{s_t \sim d_\pi^t}\left[Pr(\pi(s_t) \neq \pi^\ast(s_t))\right]$, such that
with probability $(1-\epsilon_{c})$, $\pi$ chooses the same action
as $\pi^\ast$.

As we are bounding the error of a general framework without making
specific assumptions about the task predictor and the cost function, we
assume bounds on the following variables; however, we discuss the range
of these values at the end of this section.  Given any information set,
we denote by $\Delta_{\max}$ the maximum difference in task loss due to
changing one piece of information (a insertion, deletion or substitution).
Further, we let $Q^\ast_\max$ be the maximum cost-to-go from any state
of the reference policy, and $C$ be the maximum acquisition cost of one
piece of information.

With the above definitions, we have the following guarantee for active
information selection:
\begin{corollary}
If the returned policy has error rate $\epsilon_c$ when evaluated in the
multiclass classification setting, as an information selector it satisfies
\begin{equation*}
J(\pi) - J(\pi^\ast) \le T\delta,
\end{equation*}
where $\delta = \epsilon_c \left( \Delta_{\max}+\lambda C + (1-\alpha^{-1})Q^\ast_\max \right)$.
\end{corollary}
\begin{proof}
Let $\Gamma^\pi(s,a)$ be the final information set obtained by executing $a$ in $s$ and then following $\pi$.
Now consider an auxiliary policy $\pi^{\text{aux}}$ whose actions only depend on $t$: it copies the action given by $\pi^\ast$ at the same time step after $t$ regardless of its own state. 
Let $a^\ast_t=\arg\min_a Q^{\pi^\ast}(s_t,a)$, 
$a_t=\pi(s_t)$.
The trajectories of $\Gamma^{\pi^\ast}(s_t,a^\ast_t)$ and $\Gamma^{\pi^\ast}(s_t,a_t)$ diverge from time $t$ when $a_t \neq a^\ast_t$.
Therefore starting from $s_t$, the final information sets obtained by $\pi^{\text{aux}}$ and $\pi^\ast$ differ by one element only due to $a_t \neq a^\ast$.
We use $\widetilde{\Gamma}^\pi(s,a,a')$ to denote the information set obtained by $\pi^{\text{aux}}$ copying $\pi$, 
which replaces information acquired by $a$ with that by $a'$ in $\Gamma^\pi(s,a)$. 
Therefore we have $\Gamma^{\pi^{\text{aux}}}(s_t, a_t)=\widetilde{\Gamma}^{\pi^\ast}(s_t,a^\ast_t,a_t)$.

Now we can write the $Q$-function as the loss in the terminal state.
To simplify notation, we use $\ell(\mathcal{X}', y)$ as a shorthand for $\ell(h(\mathcal{X}'), y, \mathcal{X}')$;
and similarly, $\ell_{\text{task}}(\mathcal{X}', y)$ for $\ell_{\text{task}}(h(\mathcal{X}'), y)$.
For the $i$th example we have
\begin{eqnarray}
Q^{\pi^\ast}(s_t,\pi) &=& \ell\left( \Gamma^{\pi^\ast}(s_t, a_t), y_i\right) \nonumber\\
&\le& \alpha \ell\left( \Gamma^{\pi^{\text{aux}}}(s_t, a_t), y_i \right) \nonumber\\
&=& \alpha \ell \left(\widetilde{\Gamma}^{\pi^\ast}(s_t, a_t^\ast, a_t), y_i\right ) .
\label{eqn:q-aux}
\end{eqnarray}
The inequality is due to the definition of suboptimality of $\pi^\ast$. 
Further, we have~\footnote{We omit $s,a$ in $\Gamma$ when obvious from the context.}
\begin{dmath*}
Q^{\pi^\ast}(s_t,\pi) - \min_a Q^{\pi^\ast}(s_t,a) \\
= \frac{1}{\alpha}Q^{\pi^\ast}(s_t,\pi) - Q^{\pi^\ast}(s_t,a^\ast) 
+ \left(1-\frac{1}{\alpha}\right) Q^{\pi^\ast}(s_t,\pi) \\
\le \ell \left(\widetilde{\Gamma}^{\pi^\ast}(a_t), y_i\right ) - Q^{\pi^\ast}(s_t,a^\ast)
+ \left(1-\frac{1}{\alpha}\right)Q^\ast_\max \\
= \ell_{\text{task}}\left(\widetilde{\Gamma}^{\pi^\ast}(a_t), y_i\right)
- \ell_{\text{task}}\left(\Gamma^{\pi^\ast}, y_i\right) 
+ \lambda \left(\mathcal{C}(\widetilde{\Gamma}^{\pi^\ast}(a_t))-\mathcal{C}(\Gamma^{\pi^\ast})\right) 
+ \left(1-\frac{1}{\alpha}\right)Q^\ast_\max \\
\le \Delta_\max + \lambda C + \left(1-\frac{1}{\alpha}\right)Q^\ast_\max.
\end{dmath*}
The first inequality is from Equation~\ref{eqn:q-aux}.
In the last step, the difference between task loss due to one-step deviation is bounded by $\Delta_\max$ by definition;
similarly, their costs differ by one element only which is $C$ at maximum.
To concisely present our result, below we denote the RHS (a constant) of the above inequality by $K$.

Finally, substituting Q-values in $\epsilon_N$ from Theorem~\ref{thm:l2s} with the above results, we obtain
\begin{dmath*}
\sum_{t=1}^T \mathbb{E}_{s_t \sim d_\pi^t} \left[ 
Q^{\pi^\ast}(s_t,\pi) - \min_a Q^{\pi^\ast}(s_t,a)
\right] \\
={\sum_{t=1}^T \mathbb{E}_{s_t \sim d_\pi^t} \left[ 
{Pr(a_t \neq a^\ast_t)} \left (
Q^{\pi^\ast}(s_t,\pi) - Q^{\pi^\ast}(s_t,a^\ast)
\right )
\right]} \\
\le {K \sum_{t=1}^T \mathbb{E}_{s_t \sim d_\pi^t} \left[{Pr(a_t \neq a^\ast_t)}\right] = TK\epsilon_c} .
\end{dmath*}
Therefore, from Theorem~\ref{thm:l2s} we have \[
\delta = \frac{1}{NT}\sum_{i=1}^N\sum_{t=1}^T TK\epsilon_c = K\epsilon_c.
\]
\end{proof}

\begin{figure*}
\begin{tabular}{p{0.55\textwidth}p{0.45\textwidth}}
\begin{minipage}{0.55\textwidth}
\begin{center}
\includegraphics[width=0.9\textwidth]{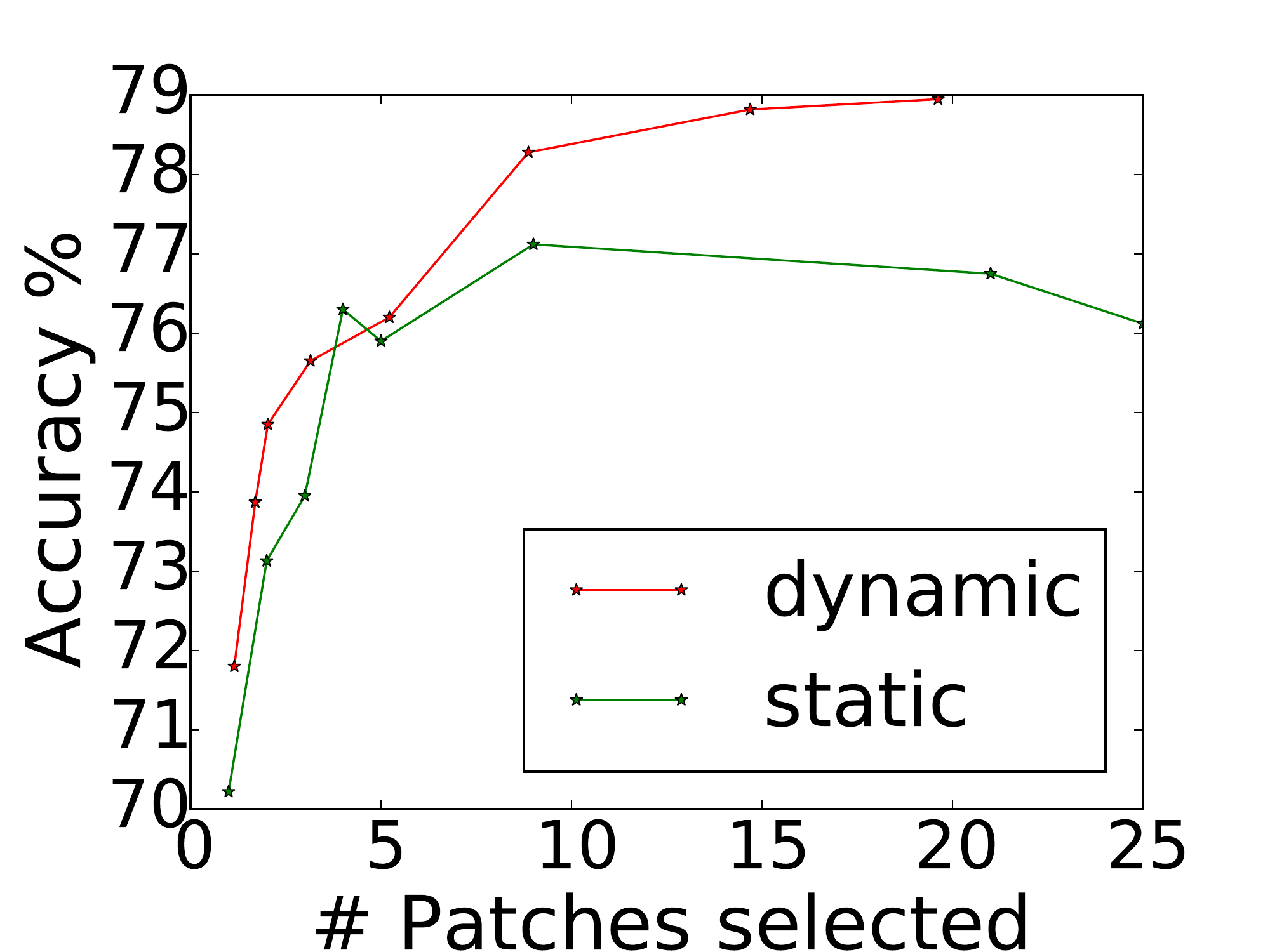}
\end{center}
\end{minipage}
&
\begin{minipage}{0.45\textwidth}
\begin{center}
\hspace{-6em}
\includegraphics[width=0.9\textwidth]{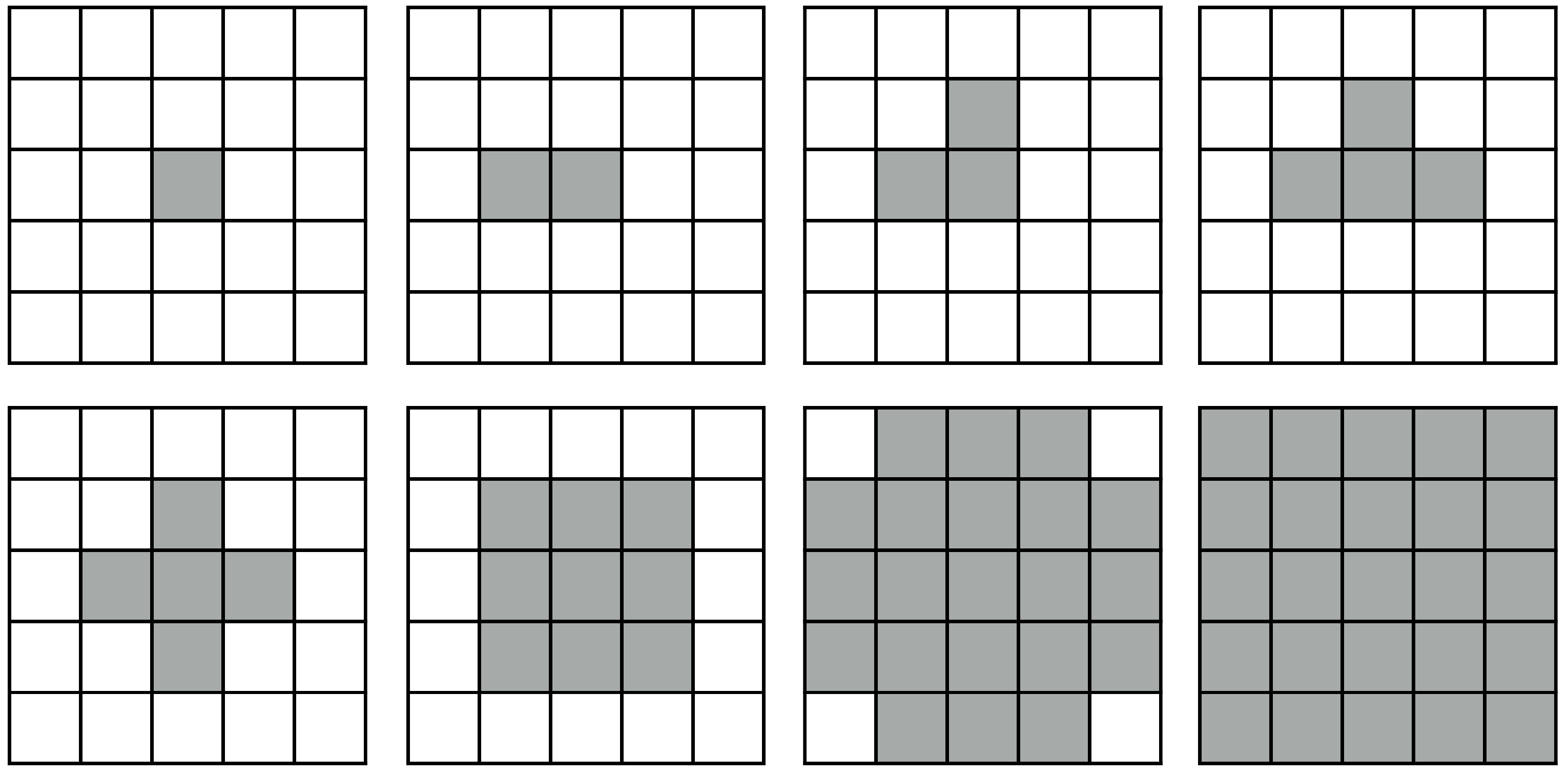}
\hspace{6em}
\end{center}
\end{minipage}
\end{tabular}
\caption{\emph{Left}: TB;DL Pareto frontiers of AIA
(dynamic) and static selection. \emph{Right}: patches selected by the
static baseline (in gray). 
}
\label{fig:tbdl-baseline}
\end{figure*}

\paragraph{Discussion}
In practice, $\alpha$ is often close to 1.  For example, if $\ell$ is a
matroid defined on $\mathcal{X}$, meaning that each part contributes to
the loss independently, then the greedy reference policy is optimal and
$\alpha=1$.  If $\ell$ is a monotone, submodular, non-negative function,
the greedy reference policy has suboptimality bounded by $(1-e^{-1})^{-1}
\approx 1.58$.  In the simple case where the cost function $\mathcal{C}$
measures the cardinality of a information set, we have $C=1$.  The maximum
cost-to-go $Q^\ast_\max$ is small when the state is on the trajectory of
$\pi^\ast$; otherwise it depends on the how well $\pi^\ast$ can recover
from a bad state.  In cases where adding information monotonically
improves the result---as we will see in the experiments--- $\pi^\ast$
can recover fast by selecting useful information even if some less
distinctive ones were added.

Therefore, the performance of our algorithm is mainly affected by two
factors.  The first is the classification error $\epsilon_c$ of $\pi$.
Given enough examples ($N \rightarrow \infty$), this is solely restricted
by the policy class $\Pi$ and the feature representation of states,
suggesting a richer policy class may work better.  The second is the
robustness of the task predictor $h$ to slight change in received
information, affecting $\Delta_{\max}$.  This can be addressed by
pre-training on randomly sampled subsets and by fine-tuning $h$ with
partial inputs induced by the learned policy $\pi$.

%% file: experiments.tex
\section{Experiments}
\label{sec:experiments}
We evaluate our algorithm AIA on two tasks with different information
sets and task classifiers: sentiment analysis and object recognition.
We show that AIA consistently performs better than the static selection
baseline.  Furthermore, it achieves a good trade-off between cost and
accuracy by acquiring more on hard examples than on easy examples.

All of our implementation is based on Vowpal
Wabbit~\cite{langford2007},\footnote{\url{http://hunch.net/~vw}}, a fast
learning system that supports online learning and L2S.  Unless stated
otherwise, we run L2S for 2 passes over the training data; fine-tuning
the predictor starts at the end of the first pass.

\subsection{TL;DR: Sentiment Analysis of Book Reviews}

In this experiment the task is to 
predict a user's rating by reading their reviews sentence
by sentence from the beginning.  We use sentences as the units of
information.  The model dynamically decides whether to continue reading
the next sentence or to stop and output the current predicted rating,
hence we refer to it as TL;DR (``Too Long; Didn't Read'').

We evaluate TL;DR on book reviews from the Amazon product
data~\cite{amazon-data}, where each review has an associated rating
between 1 and 5 inclusive.  We select reviews with 5 to 10 sentences
and split the dataset into three sets: 1M for pre-training the task
predictor, 8M for L2S and fine-tuning and 1M for testing.  Our task
predictor is a linear multiclass classifier using unigrams and
bigrams features of tokenized text.  We pre-train the predictor
on complete reviews and all prefixes.

Our information selector is a quadratic multiclass classifier.
The features are the intermediate
scores (negative log likelihood) for each class as given by the task
predictor; the difference between the highest and the next-highest score,
i.e. the score margin; the KL-divergence between the current scores and
the class prior\footnote{The prior class distribution is imbalanced in
this dataset: more than 50\% reviews have a rating of 5.}; the current
prediction of the task predictor, i.e., the argmax of the scores; and
the number of sentences read so far.

We sweep over $\lambda$ to obtain a range of models that reads different
numbers of sentences on average.  Larger $\lambda$ discourages the model
to use more information.  We compare performance of our dynamic model
with a baseline static model given various fixed amounts of information.
Our baseline model always selects the first $k$ sentences ($k \in
[5, 10]$), and utilizes a task predictor trained on the first $k$
sentences using all the examples L2S uses as training data (i.e., both
the pre-training and fine-tuning data sets).  We report macro-F1 versus
the average percentage of sentences read in Figure~\ref{fig:tldr-pareto}
and our model completely dominates the static selection method.

To examine where the model decides to acquire more information, we
compute the average percentage of sentences for each rating.
We took four models with different $\lambda$s and plot the result in Figure~\ref{fig:tldr-pareto} (right).
As $\lambda$ increases, the model reads fewer sentences on average since the penalty on cost becomes higher.
In addition, the model reads much fewer sentences
for the easy rating-5 (a majority class in our dataset) reviews and more for confusing reviews
in the middle. This shows that the model learns to acquire information
adaptively according to example difficulty.

\subsection{TB;DL: Image Recognition}

In this experiment the goal is to recognize objects by looking at a few
patches from an image.  This scenario is a toy version of a robot/camera
trying to making sense of a scene by deciding where to focus. Our model
starts from an empty image and adaptively selects a sequence of patches
to examine until it feels confident about the prediction and stops.
We refer to the model as TB;DL (``Too Big; Didn't Look'').


We evaluate our algorithm on an image classification task from PASCAL VOC
Challenge 2007.  We resize all images to $256 \times 256$.  Each image is
divided into 25 equal-sized square patches, where each patch is a part.
Our task predictor takes features extracted from the selected patches
and predicts the objects in the image.  There are 21 object classes
including the background.  For simplicity, we focus on the task of
predicting whether a person is in the image (the majority class that
often co-occurs with other classes).  To obtain patch features, we label
each patch with its image (multi-)label and fine-tune the pre-trained
VGG-16~\cite{vgg16} model from Caffe with the patch examples.  We use
the predicted probabilities output by the softmax layer of VGG network
as the patch features.\footnote{We have also tried to use features from
the penultimate fully-connected layer but found it was not helpful.}
The state features are based on intermediate scores, similar to TL;DR.

We compare against static selectors that always select a fixed
subset of patches.  As it is computationally expensive to enumerate
all possible subsets, we heuristically selected a family of subsets
that cover the image from the center to the outer parts, as shown in
Figure~\ref{fig:tbdl-baseline} (right). We obtain similar results to
the sentiment analysis task: active information acquisition shows a
better trade-off than static selection.  In fact, the static baseline
eventually shows degradation when shown larger portions of the image.
We speculate this is because VOC images often contain multiple, scattered
objects with background clutter.  Under such conditions, a limited static
focus might be better than a larger one, but a dynamic focus is best.
This supposition is supported by our heat map experiment.

\begin{figure}[t]
\setlength\tabcolsep{0pt}
\begin{tabular}{cc}
\begin{minipage}{0.45\columnwidth}
\begin{center}
\includegraphics[width=0.9\columnwidth]{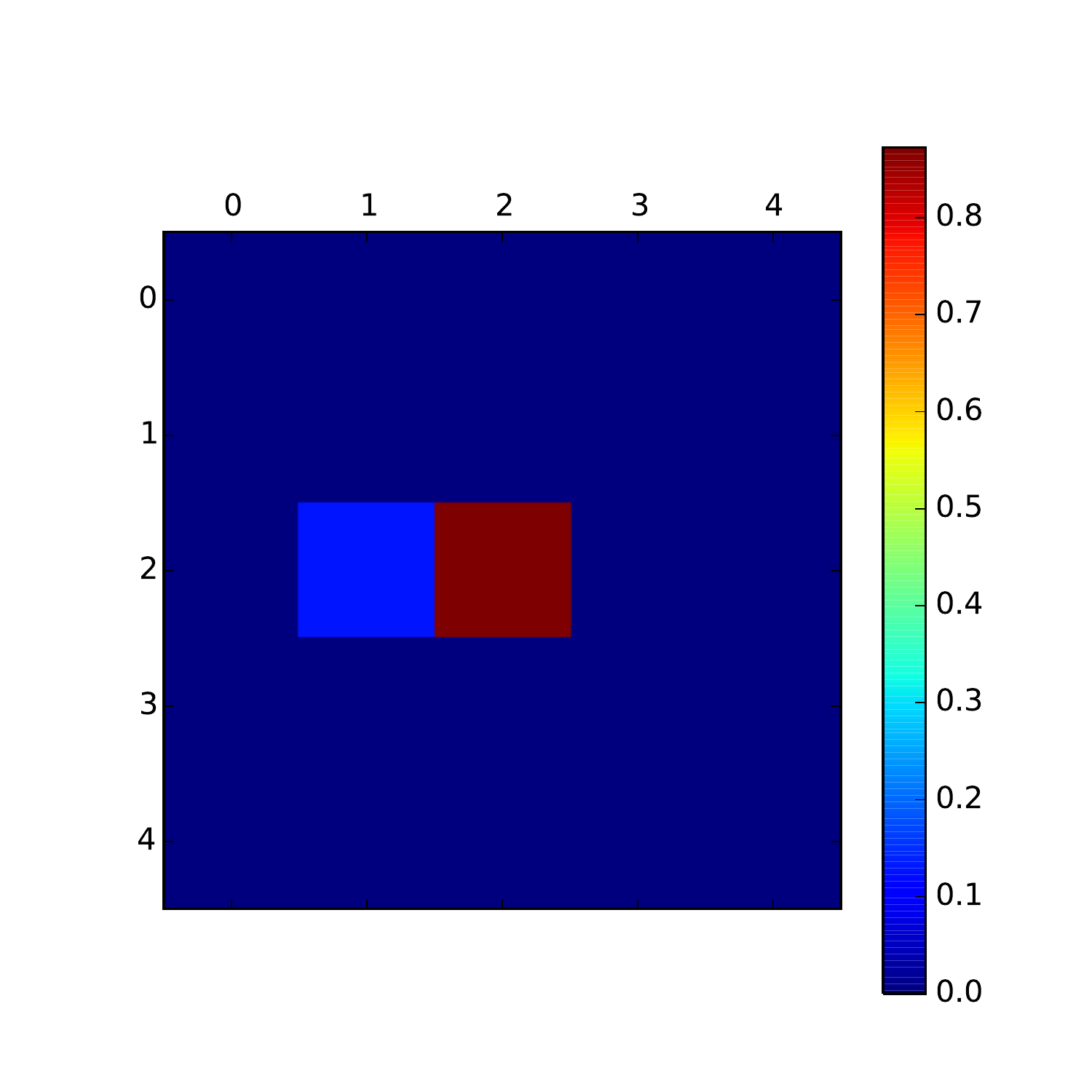}
\end{center}
\end{minipage}
&
\begin{minipage}{0.45\columnwidth}
\begin{center}
\includegraphics[width=0.9\columnwidth]{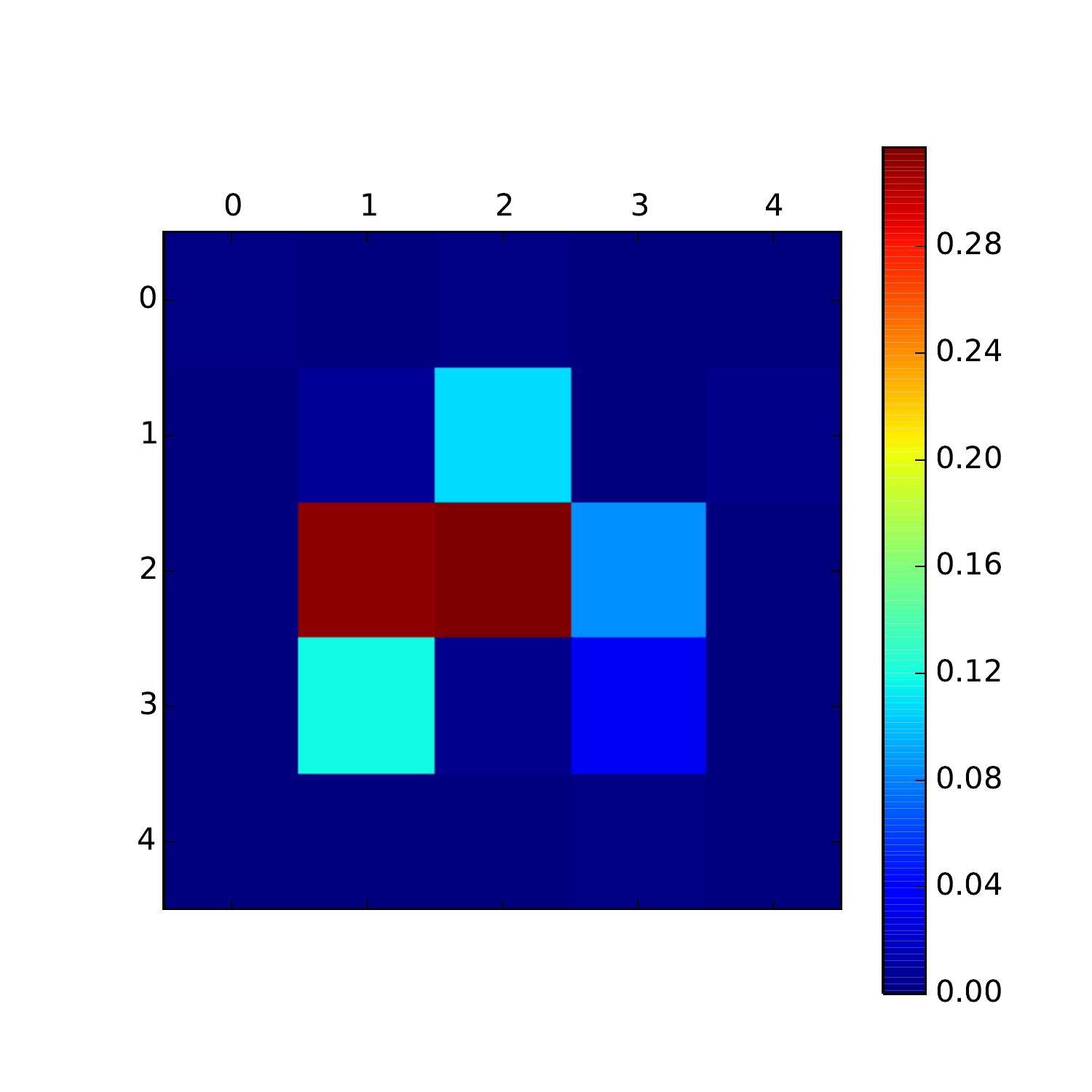}
\end{center}
\end{minipage}
\\ 
\begin{minipage}{0.45\columnwidth}
\begin{center}
\includegraphics[width=0.9\columnwidth]{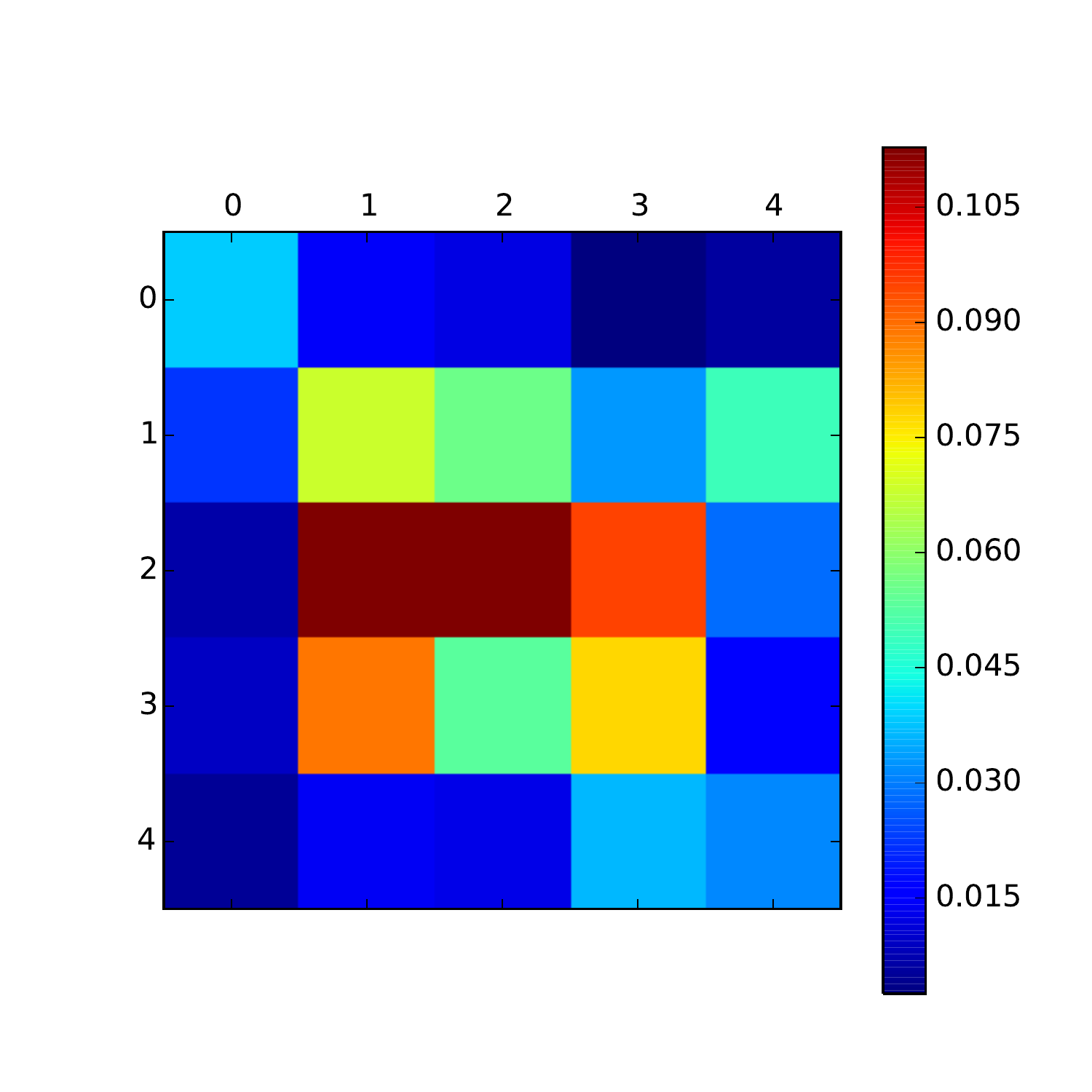}
\end{center}
\end{minipage}
&
\begin{minipage}{0.45\columnwidth}
\begin{center}
\includegraphics[width=0.9\columnwidth]{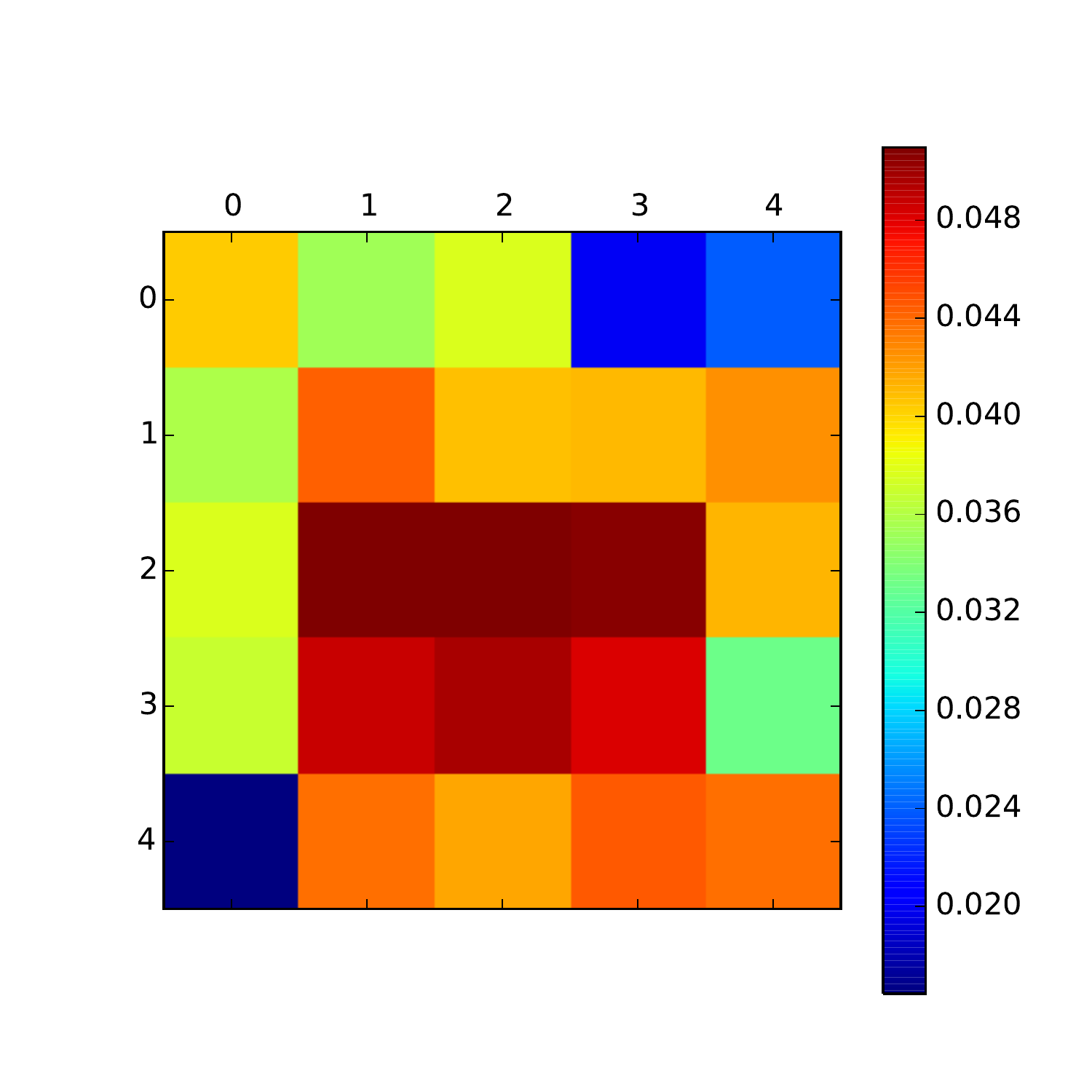}
\end{center}
\end{minipage}
\end{tabular}
\caption{Heat maps of frequencies a patch get selected at different $\lambda$. Decreasing $\lambda$ implies information is less costly to acquire.  \emph{Top Left}: $\lambda = 3.5$.  \emph{Top Right}: $\lambda = 2.0$.  \emph{Bottom Left}: $\lambda = 1.0$.  \emph{Bottom Right}: $\lambda = 0.0$. Best viewed in color.}
\label{fig:heatmap}
\end{figure}

To examine where the model pays most attention, we show heat maps of the
attention of models with different trade-offs in Figure~\ref{fig:heatmap}
(best viewed in color).
The result is consistent with our intuition: when the amount of
information is restricted, the learned policy looks mostly in the center 
where the object is more likely to be located; when more information is 
allowed, the policy dynamically explores outer parts.  
Furthermore, when information acquisition is free,
i.e., when $\lambda = 0$, the model still chooses to classify before
viewing the entire image, indicating a limited static focus can be
beneficial even absent acquisition costs.

\begin{figure}[ht]
\centering
\includegraphics[width=0.9\columnwidth]{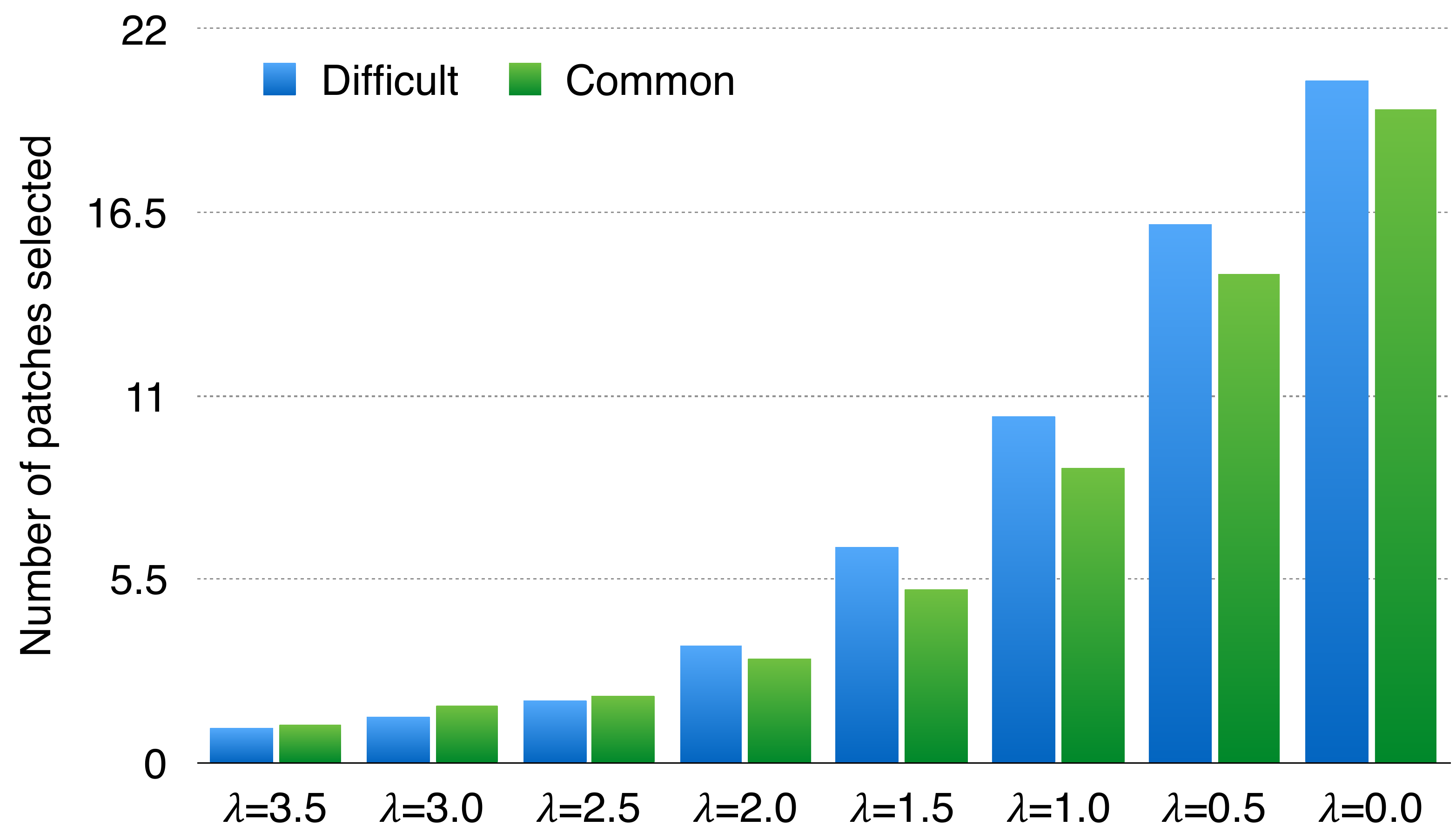}
\caption{Average number of patches selected by AIA for hard and easy examples. The selector looks at more patches on difficult examples.}
\label{fig:tbdl-hist}
\end{figure}

The VOC dataset also contains annotations about hard instances, which
we use to confirm that the model learns to use more information for
hard examples.  In Figure~\ref{fig:tbdl-hist}, we report the average number of patches selected for both hard and easy examples. 
When $\lambda$ is large, the policy selects approximately the same number of patches for both types of images, since the cost penalty does not allow for more exploration.
When the constraint on cost is relaxed, we see that
for difficult images the average number of patches
selected is consistently larger than that for common images.